\documentclass[11pt]{article}
\usepackage[margin=1in]{geometry}
\usepackage[utf8]{inputenc}
\usepackage[T1]{fontenc}

\usepackage{fancyhdr}

\fancypagestyle{plain}{
  \fancyhf{}
  \cfoot{\small\thepage}

}

\usepackage{amsmath, amssymb, amsfonts, amsthm}
\usepackage{mathtools}

\newtheorem{theorem}{Theorem}

\newtheorem{definition}{Definition}
\newtheorem{assumption}{Assumption}

\newtheorem{lemma}{Lemma}

\usepackage{graphicx}
\graphicspath{{./images/}}
\usepackage{booktabs}
\usepackage{array}
\usepackage{multirow}
\usepackage{subfigure}

\usepackage{algorithm}
\usepackage{algpseudocode}

\usepackage{microtype}
\usepackage{nicefrac}
\usepackage{textcomp}
\usepackage{xcolor}
\usepackage{balance}

\usepackage{cite}
\usepackage{url}
\usepackage{hyperref}
\usepackage{cleveref}
\crefname{theorem}{Theorem}{Theorems}
\Crefname{theorem}{Theorem}{Theorems}

\crefname{corollary}{Corollary}{Corollaries}
\Crefname{corollary}{Corollary}{Corollaries}

\crefname{definition}{Definition}{Definitions}
\Crefname{definition}{Definition}{Definitions}

\crefname{assumption}{Assumption}{Assumptions}
\Crefname{assumption}{Assumption}{Assumptions}

\crefname{proposition}{Proposition}{Propositions}
\Crefname{proposition}{Proposition}{Propositions}

\crefname{lemma}{Lemma}{Lemmas}
\Crefname{lemma}{Lemma}{Lemmas}
\hypersetup{
    colorlinks=true,
    linkcolor=red,   
    citecolor=blue,    
    urlcolor=magenta  
}

\providecommand{\keywords}[1]{%
  \vspace{0.5em}
  \noindent\textbf{Keywords---}#1
}

\newcommand{\R}{\mathbb{R}}

\def\BibTeX{{\rm B\kern-.05em{\sc i\kern-.025em b}\kern-.08em
T\kern-.1667em\lower.7ex\hbox{E}\kern-.125emX}}

\begin{document}

\title{\textbf{Achieving $\widetilde{\mathcal{O}}(1/\epsilon)$ Sample Complexity for Bilinear Systems Identification under Bounded Noises}}

\author{
Hongyu Yi\thanks{Department of Electrical and Computer Engineering, University of Washington, Seattle, WA, USA. } , Chenbei Lu\thanks{Cornell University AI for Science Institute, Cornell University, Ithaca, NY, USA.} , Jing Yu\footnotemark[1]
}
\date{\normalsize \today}

\maketitle

\begin{abstract}
This paper studies finite-sample set-membership identification for discrete-time bilinear systems under bounded symmetric log-concave disturbances. Our analysis considers trajectory-dependent regressors and allows marginally stable dynamics with polynomial mean-square state growth. We prove that the diameter of the feasible parameter set shrinks with sample complexity $\widetilde{\mathcal O}(1/\epsilon)$ where $\epsilon$ is the estimation error. Simulation supports the theory and illustrates the advantage of the proposed estimator for uncertainty quantification.
\end{abstract}


\section{Introduction}
Learning dynamical systems from data, i.e., system identification, is a central problem in control, with broad impact across robotics, aerospace, and energy systems. In recent years, quantifying the sample efficiency of system identification algorithms has gained increasing interest.

In particular, a large body of recent work studies the non-asymptotic sample complexity of identifying linear and switched linear dynamical systems \cite{sarkar2019Nearoptimalfinite,Oymak2019Non-asymptoticIdentification,Zeng2025SystemIdentificationUnderBoundedNoise,Racz2025LSS}. Most work builds on least-squares-type estimators under standard stochastic noise models (e.g., Gaussian noise), establishing the canonical sample complexity of $\widetilde{\mathcal{O}}(1/\epsilon^{2})$\footnote{$f(\epsilon)=\widetilde{\mathcal{O}}(g(\epsilon))$ means
$f(\epsilon)=\mathcal{O}(g(\epsilon)\log^c(1/\epsilon))$ for some $c>0$.} where it takes at least $\widetilde{\mathcal{O}}(1/\epsilon^2)$ number of data points to achieve an estimation error of $\epsilon$. These results enable learning-based control with provable guarantees \cite{Simchowitz2018LearningWithoutMixing,simchowitz2020NaiveExploration}. 
More recently, it has been shown that when the noises are bounded, tail-driven least-squares analyses can be overly conservative. In particular, \cite{Zeng2025SystemIdentificationUnderBoundedNoise} shows that with \emph{bounded} noise, the minimax-optimal sample complexity for linear system identification improves to $T=\widetilde{\mathcal{O}}(1/\epsilon)$, revealing a sharp gap between bounded and unbounded noise regimes and confirming that the canonical $\widetilde{\mathcal{O}}(1/\epsilon^{2})$ scaling is suboptimal in bounded-noise settings.

Set-membership estimation (SME) \cite{bai1998convergenceproperties-sme, akccay2004size, milanese2004set} has recently gained renewed attention for system identification \cite{xu2025sample,gao2024closure, lauricella2020set} and learning-based control \cite{niknejad2025online, yu2023online,parsi2023dual} under bounded noise: by directly constructing uncertainty sets consistent with the observed data, SME can explicitly exploit bounded-noise structure rather than relying on stochastic tail assumptions. 
More recently, for linear dynamical systems, non-asymptotic analyses show that the sample complexity of the SME can achieve the optimal $\widetilde{\mathcal{O}}(1/\epsilon)$ rate under bounded noise \cite{Li2024SME-LTI}. 
These developments underscore the growing relevance of set-based identification in safety-critical settings.

Beyond linear dynamics, bilinear systems provide a broad and practically relevant model class that captures rich multiplicative coupling between the state and the input, making them more expressive for many control applications including energy systems \cite{Sattar2022SysId,KELMAN2011BilinearHVAC}. At the same time, this state-action coupling propagates through time and induces strong temporal dependence between regressors and noise, which substantially complicates identification and excitation in the finite-sample regime. 
Existing results primarily focus on least-squares-type \emph{point} estimation for bilinear systems \cite{Sattar2022SysId,sattar2025finitesampleidentificationpartially}, leaving finite-sample \emph{uncertainty set-based} guarantees under bounded noise largely open. A closely related result is \cite{yingying2024NonlinearDynamicalSystems}, where $\widetilde{\mathcal{O}}(1/\epsilon)$ sample complexity was proved for general nonlinear analytic systems without explicit dimension dependence under the assumption that the system has deterministic local input-to-state stability (LISS). In contrast, our analysis provides explicit dimension dependence and allows marginal mean-square stability with polynomially growing~regressors. 

In this paper, we develop a finite-sample set-membership identification framework for discrete-time bilinear systems with bounded symmetric log-concave noise. We consider single-trajectory settings where the bilinear regressors may grow polynomially over time. Our arguments involve a weaker log-concave noise assumption that precludes direct use of standard Gaussian tail arguments. The main technical contribution of this paper is a finite-sample analysis that controls the growth of the state-input features and establishes the persistent excitation conditions under these weaker stability and noise assumptions, leading to an explicit contraction guarantee for the SME feasible parameter set.

\section{Problem Setup}
\label{sec:problem-def}
In this section, we present the bilinear system model and introduce the set-membership identification algorithm, along with assumptions.
\subsection{Bilinear Dynamical System Model}
We consider the discrete-time bilinear system
\begin{align}
\mathbf{x}_{t+1}
= \mathbf{A}\mathbf{x}_{t}
+ \sum\nolimits_{i=1}^{m}\mathbf{u}_{t}[i]\mathbf{B}_{i}\mathbf{x}_{t}
+ \mathbf{w}_{t},
\label{eq:system}
\end{align}
where $\mathbf{x}_{t}\in\mathbb{R}^{n}$ is the state and $\mathbf{u}_{t}\in\mathbb{R}^{m}$ is the input, with $\mathbf{u}_t[i]$ denoting the $i$th coordinate of $\mathbf{u}_{t}$. The initial state $\mathbf{x}_0$ satisfies $\mathbb{E}[\|\mathbf{x}_0\|_2^2] < \infty$, and is independent of $\mathbf{u}_t \text{ and } \mathbf{w}_t$ for $t \not = 0$. The matrices
$\mathbf{A}\in\mathbb{R}^{n\times n}$ and $\{\mathbf{B}_{i}\in\mathbb{R}^{n\times n}\}_{i=1}^{m}$ are unknown and must be identified.
The \emph{i.i.d.} noise $\mathbf{w}_t$ takes values in a known compact set $\mathbb{W}\subset\mathbb{R}^n$, i.e., $\mathbf{w}_t\in\mathbb{W}$ for all $t$. Define the block matrix $
\mathbf{B}:=[\mathbf{B}_{1},\ldots,\mathbf{B}_{m}]\in\mathbb{R}^{n\times (nm)},
\ 
\mathbf{\Theta}_\star:=\big[\,\mathbf{A}\ \ \mathbf{B}\,\big]\in\mathbb{R}^{n\times (n+nm)}.$
Consider the regressor:
\begin{align}
\mathbf{z}_{t}
:=
\begin{bmatrix}
\mathbf{x}_{t}^\top \quad
(\mathbf{u}_{t}\otimes \mathbf{x}_{t})^\top
\end{bmatrix}^\top
\in \mathbb{R}^{n+nm}.
\label{eq:regressor}
\end{align}
Then \eqref{eq:system} can be written compactly as
\begin{align}
\mathbf{x}_{t+1}=\mathbf{\Theta}_\star\mathbf{z}_{t}+\mathbf{w}_{t}\, .
\label{eq:lin-param}
\end{align}
We assume access to a length-$T$ trajectory $\mathcal{D}_{T}
:=\{(\mathbf{z}_{t},\mathbf{x}_{t+1})\}_{t=0}^{T-1}$ generated by \eqref{eq:lin-param}. The identification goal is to learn the unknown parameter matrix $\mathbf{\Theta}_\star$ consistent with the data and the noise set $\mathbb{W}$.

\subsection{Set-membership Identification}
SME constructs the \emph{feasible set} of parameters that are consistent with the observed trajectory and the known noise set $\mathbb{W}$. 
Given the dataset $\mathcal{D}_T=\{(\mathbf{z}_t,\mathbf{x}_{t+1})\}_{t=0}^{T-1}$ generated by \eqref{eq:lin-param} and the bound $\mathbf{w}_t\in\mathbb{W}$, define
\begin{align} \mathbb{S}_{T} :=\bigcap\nolimits_{t=0}^{T-1} \left\{ \mathbf{\Theta}\in\mathbb{R}^{n\times(n+nm)}: \mathbf{x}_{t+1}-\mathbf{\Theta}\mathbf{z}_{t}\in\mathbb{W} \right\}. \label{eq:SME-set} \end{align}
Equivalently, $\mathbf{\Theta}\in\mathbb{S}_{T}$ if and only if there exists a noise sequence $\{\tilde{\mathbf{w}}_t\}_{t=0}^{T-1}$ with $\tilde{\mathbf{w}}_t\in\mathbb{W}$ such that
$\mathbf{x}_{t+1}=\mathbf{\Theta}\mathbf{z}_t+\tilde{\mathbf{w}}_t$ holds for all $t=0,\ldots,T-1$.
By construction, $\mathbb{S}_{T}$ is nested: $\mathbb{S}_{T+1}\subseteq \mathbb{S}_{T}$ and $\mathbf{\Theta}_\star \in \mathbb{S}_T$ for all $T \geq 0$.  We quantify the size of the feasible set produced by SME with the diameter of $\mathbb{S}_T$.
\begin{definition}[Diameter of a set]
For a bounded set $\mathbb{S}$, its diameter is $\mathrm{diam}(\mathbb{S})
:=
\sup\nolimits_{\mathbf{X},\mathbf{Y}\in\mathbb{S}}
\|\mathbf{X}-\mathbf{Y}\|_{F}$, where $\|\cdot\|_{F}$ denotes the Frobenius norm. If $\mathbb{S}$ is unbounded, then $\mathrm{diam}(\mathbb{S}) = \infty$.
\end{definition}
Our goal is to quantify how quickly the diameter of $\mathbb{S}_T$ decreases as $T$ grows.

\subsection{Assumptions and Definitions}

We impose standard conditions on the input process and the disturbance to facilitate a finite-sample analysis of \eqref{eq:SME-set}.

\begin{assumption}[Input process]\label{ass:input}
The input sequence $\{\mathbf{u}_{t}\}_{t\ge 0}$ is \emph{i.i.d.}, satisfies $\mathbb{E}[\mathbf{u}_{t}]=\mathbf{0}$ and $\|\mathbf{u}_{t}\|_{\infty}\le u_{\max}$ almost surely, and has covariance
$\mathbb{E}[\mathbf{u}_{t}\mathbf{u}_{t}^{\top}]=\sigma_{u}^{2}\mathbf{I}_{m}$ for some $\sigma_{u}>0$. 
Moreover, $\mathbf{u}_{t}$ is independent of $\mathbf{x}_{0}$ and of $\{\mathbf{w}_{t}\}_{t\ge 0}$.
\end{assumption}

\begin{assumption}[Bounded noise model]\label{ass:noise}
The disturbance sequence $\{\mathbf{w}_{t}\}_{t\ge 0}$ is \emph{i.i.d.}, satisfies $\mathbb{E}[\mathbf{w}_{t}]=\mathbf{0},\ \mathrm{Cov}(\mathbf{w}_{t}) \succeq \sigma_w^2 \mathbf{I}$ for a $\sigma_w >0$, and is supported on
$\mathbb{W}=\{\mathbf{w}:\|\mathbf{w}\|_{\infty}\le w_{\max}\}$.
We further assume that for any $\|\mathbf{v}\|_{2}=1$, the marginal $\mathbf{v}^{\top}\mathbf{w}_{t}$ has a centered, symmetric, log-concave density.
\end{assumption}

\begin{assumption}[Boundary mass of noise]\label{ass:tight-bound-w}
Let $\mathbb{W}=\{\mathbf{w}:\|\mathbf{w}\|_{\infty}\le w_{\max}\}$. There exist constants $c_w>0$ and $\varepsilon_0\in(0,w_{\max}]$ such that
for any $j\in[n]$, $b\in\{\pm 1\}$, and $\varepsilon\in(0,\varepsilon_0]$, $\mathbb{P}\!\left(b\,\mathbf{w}_t[j]\ge w_{\max}-\varepsilon\right)\ge c_w\,\varepsilon$ holds.

\end{assumption}

\medskip
We also introduce the following two widely used excitation notions, persistent excitation (PE) and the block martingale small-ball (BMSB) condition, which will be used to control the regressors $\{\mathbf{z}_t\}$ and derive contraction bounds.

\begin{definition}[Persistent excitation]\label{def:PE}
Let $\{\mathbf{z}_t\}_{t\ge 0}$ be the regressor sequence defined in \eqref{eq:regressor}. 
We say that $\{\mathbf{z}_t\}$ is \emph{persistently exciting} with uniform parameters $(\alpha,\kappa)$, where $\alpha>0$ and $\kappa\in\mathbb{N}_+$, if for every $t_0\ge 0$,
\begin{align*}
1/\kappa\sum\nolimits_{t=t_0}^{t_0+\kappa-1} \mathbb{E}\!\left[\mathbf{z}_t \mathbf{z}_t^\top \right]
\succeq \alpha^2 \mathbf{I}_{n+nm}.
\end{align*}
\end{definition}
\vspace{0.6\baselineskip}

\begin{definition}[BMSB condition]\label{def:BMSB}
Consider a filtration $\{\mathcal{F}_{t}\}_{t\ge 1}$ and an $\{\mathcal{F}_{t}\}$-adapted process $\{Z_{t}\}_{t\ge 1}$ in $\mathbb{R}^{d}$.
We say $\{Z_{t}\}$ satisfies the $(k,\Gamma_{\mathrm{sb}},p)$-block martingale small-ball condition if there exist uniform $k\in\mathbb{N}_{+}$, $p\in(0,1]$, and $\Gamma_{\mathrm{sb}}\succ 0$, such that for any unit vector $\lambda\in\mathbb{R}^{d}$,
\begin{align*}
1/k\sum\nolimits_{i=1}^{k}
\mathbb{P}\!\left(\big|\lambda^{\top}Z_{t+i}\big|\ge \sqrt{\lambda^{\top}\Gamma_{\mathrm{sb}}\lambda}\ \middle|\ \mathcal{F}_{t}\right)
\ge p,
\quad \forall t\ge 1\,.
\end{align*}
\end{definition}

\section{Finite Sample Analysis}

Our first contribution is defining a quantitative notion of marginal stability of bilinear systems. This follows the classic work~\cite{Kubrusly1985MSS} on marginal stability and enters the sample-complexity bound in \Cref{thm:sample_complexity}.

\begin{lemma}[Polynomial mean-square growth]\label{lem:PMS}
Consider \Cref{ass:input,ass:noise}, and let the initial state satisfy $\mathbb{E}[\|\mathbf{x}_0\|_2^2] < \infty$ and be independent of $\mathbf{u}_t, \mathbf{w}_t$ for all $t\not = 0$. Define the augmented matrix $\tilde{\mathbf A}$ that takes the form
\begin{align*}
\tilde{\mathbf A}
  := \mathbf F \otimes \mathbf F
   + \sum\nolimits_{k=1}^m \sum\nolimits_{\ell=1}^m
     \gamma_{k\ell}\, \mathbf B_\ell \otimes \mathbf B_k ,
\end{align*}
with $\mathbf F := \mathbf A + \sum_{k=1}^m  \mathbb{E}[\mathbf{u}_t[k]]\, \mathbf B_k$ and
$\gamma_{k\ell} := \mathbb{E}[\mathbf{u}_t[k]\mathbf{u}_t[\ell]]- \mathbb{E}[\mathbf{u}_t[k]]\,\mathbb{E}[\mathbf{u}_t[\ell]]$. If the augmented matrix $\tilde{\mathbf A}$ satisfies $\rho(\tilde{\mathbf A})\le 1$, the bilinear system exhibits \emph{polynomial mean-square growth}: there exist constants $c^{\mathrm{PMS}}<\infty$ and an integer $r\in \{0,1,..,n^2\}$ such that $\mathbb{E}\!\left[\|\mathbf{x}_t\|_2^2\right]\le c^{\mathrm{PMS}}(1+t^r),\ \forall t\ge 0$.

\end{lemma}

\begin{proof}
Define the state variance matrix $\Sigma_t:=\mathbb E[\mathbf{x}_t \mathbf{x}_t^\top]$.
Expanding $\mathbf{x}_{t+1} \mathbf{x}_{t+1}^\top$ and taking expectations yields
\begin{equation*}
\Sigma_{t+1}
=
\mathbf F \Sigma_t \mathbf F^\top
+
\sum\nolimits_{k=1}^m\sum\nolimits_{\ell=1}^m \gamma_{k\ell} \mathbf B_k \Sigma_t \mathbf B_\ell^\top
+
\Sigma_w,
\end{equation*}
where $\mathbf F:= \mathbf A+\sum_{k=1}^m \mathbb E[\mathbf u_t[k]] \mathbf B_k,\
\gamma_{k\ell}:=\mathbb E[\mathbf u_t[k] \mathbf u_t[\ell]]-\mathbb E[\mathbf u_t[k]]\mathbb E[\mathbf u_t[\ell]]$, and $\Sigma_w:=\mathbb E[\mathbf w_t \mathbf w_t^\top]$. Under \Cref{ass:input}, we have \(\mathbf{F}=\mathbf A\) and
\(\gamma_{k\ell}=\sigma_u^2\mathbf 1_{\{k=\ell\}}\). Hence, $\tilde{\mathbf A}$ can be written as $\tilde{\mathbf A} = \mathbf A \otimes \mathbf A + \sigma_u^2 \sum\nolimits_{k=1}^m \mathbf B_k \otimes \mathbf B_k$. Vectorizing both sides and using
$\mathrm{vec}(\mathbf M \mathbf X \mathbf N)=(\mathbf N^\top\otimes \mathbf  M)\mathrm{vec}(\mathbf X)$, we obtain
\begin{equation}
\mathrm{vec}(\Sigma_{t+1})
=
\tilde{\mathbf{A}} \,\mathrm{vec}(\Sigma_t)
+
\mathrm{vec}(\Sigma_w).
\label{eq:P2_recursion}
\end{equation}
Iterating \eqref{eq:P2_recursion} yields
\begin{equation}
\mathrm{vec}(\Sigma_t)
=
\tilde{\mathbf{A}}^t\,\mathrm{vec}(\Sigma_0)
+
\sum\nolimits_{i=0}^{t-1}\tilde{\mathbf{A}}^i\,\mathrm{vec}(\Sigma_w).
\label{eq:At-iterate}
\end{equation} 
Denote $\dim(\tilde{\mathbf{A}})$ by $ n^2$. Since \(\rho(\tilde{\mathbf{A}})\le 1\) and \(\tilde{\mathbf{A}}\in\mathbb{R}^{n^2\times n^2}\), by the Jordan decomposition, there exists a constant $C_{\tilde A}<\infty$ such that for all $t\ge 1$,
\begin{equation}
\|\tilde{\mathbf{A}}^t\|_2 \le C_{\tilde A}(1+t^{n^2-1}).
\label{eq:At-bound}
\end{equation}
Consequently,
\begin{equation}
\sum\nolimits_{i=0}^{t-1}\|\tilde{\mathbf{A}}^i\|_2
\le C'_{\tilde A}(1+t^{n^2}),
\label{eq:sum-At-bound}
\end{equation}
for some constant $C'_{\tilde A}<\infty$. Applying \eqref{eq:At-bound}--\eqref{eq:sum-At-bound} to \eqref{eq:At-iterate} gives
\begin{align*}
\|\mathrm{vec}(\Sigma_t)\|_2
&\le
\|\tilde{\mathbf{A}}^t\|_2\,\|\mathrm{vec}(\Sigma_0)\|_2
+
\sum\nolimits_{i=0}^{t-1}\|\tilde{\mathbf{A}}^i\|_2\,\|\mathrm{vec}(\Sigma_w)\|_2 \notag
\\
&\le
C_{\tilde A}(1+t^{n^2-1})\|\mathrm{vec}(\Sigma_0)\|_2
+
C'_{\tilde A}(1+t^{n^2})\|\mathrm{vec}(\Sigma_w)\|_2 \le C_\Sigma(1+t^{n^2}),
\end{align*}
for some constant $C_\Sigma<\infty$. Finally, using $\mathbb E\|\mathbf x_t\|_2^2=\mathrm{tr}(\Sigma_t)$ and
$
\mathrm{tr}(\Sigma_t)
\le \sqrt{n}\|\Sigma_t\|_F
=\sqrt{n}\|\mathrm{vec}(\Sigma_t)\|_2,
$
we obtain $\mathbb E\|\mathbf{x}_t\|_2^2
\le c^{\mathrm{PMS}}(1+t^{r}),\ \forall t$, where we can always take $r=n^2$ and $c^{\mathrm{PMS}}:=\sqrt n\,C_\Sigma$.
\end{proof}

Next, we relate the BMSB condition to persistent excitation, which enables the downstream concentration analysis.

\begin{lemma}[BMSB for $\mathbf{z}_t$]\label{lem:BMSB}
Let the filtration $\mathcal{F}_t := \sigma\big(\mathbf{x}_0,\{\mathbf{w}_s,\mathbf{z}_{s+1}\}_{s=0}^{t-1}\big)$. Then the $\{\mathcal{F}_t\}$-adapted process $\{\mathbf{z}_t\}_{t\ge 0}$ satisfies the $(1,k_z^2\mathbf{I}_{n+nm},p_z)$-BMSB condition for some constants $k_z,p_z>0$, where $\sigma(\cdot)$ denotes the sigma-algebra generated by the random variables.
\end{lemma}
\noindent
\textit{Proof Sketch. }
Compared with the Gaussian-noise settings, the proof cannot rely on explicit Gaussian tail formulas for shifted one-dimensional projections. Instead, the log-concave setting requires a separate anti-concentration argument based on symmetry and unimodality. The proof proceeds in three steps. First, for a fixed direction $\mathbf v$, we decompose $\langle \mathbf v,\mathbf z_{j+1}\rangle
= \langle \mathbf v_0+\mathbf V\mathbf u_{j+1},\,\mathbf{\Theta}_\star \mathbf z_j+\mathbf w_j\rangle$
and reduce the desired lower bound to the intersection of two events: one controlling the size of the random coefficient $\|\mathbf v_0+\mathbf V\mathbf u_{j+1}\|_2$, and the other controlling the excitation contributed by the disturbance along this random direction. Second, we lower bound the disturbance event by exploiting symmetry and log-concavity of one-dimensional projections of $\mathbf w_j$, together with the Paley--Zygmund inequality and a fourth-to-second moment comparison for symmetric log-concave random variables. Third, we lower bound the input event using another Paley--Zygmund argument and combine the two bounds to obtain the BMSB condition. The complete proof is provided in \Cref{app:lemBMSB}.

With the above technical setup, we now establish the finite-sample convergence guarantee for the SME feasible sets $\mathbb{S}_T$. 
The following result shows that the SME feasible set has sample complexity
$\tilde{\mathcal{O}}(n^{3.5}m^5/\epsilon)$ in terms of the problem dimensions under the bounded-noise~regime.

\begin{theorem}[Sample complexity guarantee]\label{thm:sample_complexity}
Suppose \Cref{ass:input,ass:noise,ass:tight-bound-w} hold, $\mathbb E\|x_0\|_2^2<\infty$, and
$\rho(\widetilde{\mathbf{A}})\le 1$. Let $k_z,p_z$ be the constants in \Cref{lem:BMSB},
let $r$ be the polynomial-growth exponent in \Cref{lem:PMS}, and let $C_z$ be
the truncation constant appearing in \Cref{lem:E2c,lem:E1E2}. For any $\eta\in(0,1)$ and any $\epsilon\in(0, \frac{16\sqrt n\,\epsilon_0}{k_zp_z}]$, define $M:=\sqrt{\frac{6C_zT^{r+1}}{\eta}}$, and $\epsilon_z
:=
\min\left\{
\frac{k_z^2p_z^2\eta}{192C_zT^{r+1}},
\frac14
\right\},\
\epsilon_\gamma
:=
\min\left\{
\frac{k_zp_z\sqrt{\eta}}{16\sqrt{6nC_zT^{r+1}}},
\frac14
\right\}$. If the number of samples $T$ satisfies 
\[T
\ge
\frac{256\sqrt n}{k_zp_z^3c_w\,\epsilon}
(1+\log\frac{1632T}{\eta}
+5(n+nm)\log\frac{n+nm}{\epsilon_z})
(1+\log\frac{1632}{\eta}
+5(n^2+n^2m)\log\frac{n^2+n^2m}{\epsilon_\gamma}),
\] 
then $\mathbb P\left(\operatorname{diam}(S_T)>\epsilon \right)\le \eta$ holds. The constants $r,C_z,k_z,p_z,c_w,\epsilon_0$ are fixed by \Cref{lem:PMS,lem:BMSB,lem:E2c,lem:E1E2}, and \Cref{ass:tight-bound-w}. 
\end{theorem}

\begin{proof}
We analyze the convergence of the SME feasible set $\mathbb{S}_T$ by working with the induced feasible set of parameter errors. Let $n_z:=n+nm$. For any candidate $\mathbf{\Theta}\in \mathbb{S}_T$, write $\gamma=\mathbf{\Theta}-\mathbf{\Theta}_\star$. Since $\mathbf{x}_{t+1}=\mathbf{\Theta}_\star \mathbf{z}_t+\mathbf{w}_t$, feasibility of $\mathbf{\Theta}$ is equivalent to $\mathbf{w}_t-\gamma \mathbf{z}_t\in \mathbb{W}$ for all $t=0,\ldots,T-1$. Hence $\mathbb{S}_T=\mathbf{\Theta}_\star+\Gamma_T$, where
\begin{align*}
\Gamma_T
:=
\bigcap\nolimits_{t=0}^{T-1}
\left\{\gamma\in\mathbb{R}^{n\times n_z}:\ \mathbf{w}_t-\gamma \mathbf{z}_t \in \mathbb{W}\right\}.
\end{align*}
Since $\mathbb{S}_T=\mathbf{\Theta}_\star+\Gamma_T$, translation invariance gives
\begin{align*}
\mathrm{diam}(\mathbb{S}_T) = \mathrm{diam}(\Gamma_T)
=\sup\nolimits_{\gamma,\gamma'\in\Gamma_T}\|\gamma-\gamma'\|_F
\le 2\sup\nolimits_{\gamma\in\Gamma_T}\|\gamma\|_F .
\end{align*}
Thus, with $\mathcal{E}_1 := \left\{\exists\gamma\in\Gamma_T: \|\gamma\|_F\ge \epsilon/2\right\}$, we have $\{\operatorname{diam}(\mathbb{S}_T)>\epsilon \} \subseteq \mathcal{E}_1$ and therefore $\mathbb{P}(\operatorname{diam}(\mathbb{S}_T)>\epsilon)\le \mathbb{P}(\mathcal{E}_1)$.

Now let $d_\gamma:=n^2+n^2m$, $L_z:=1+\log\frac{1632T}{\eta}+5n_z\log\frac{n_z}{\epsilon_z},\ L_\gamma:=1+\log\frac{1632}{\eta}+5d_\gamma\log\frac{d_\gamma}{\epsilon_\gamma}$, and $a_1:=k_zp_z/4$. Choose an integer $\kappa$ such that $8L_z/p_z^2\le \kappa\le 16L_z/p_z^2$. For convenience, we can always take $\kappa=\lceil 8L_z/p_z^2\rceil$. Define the block PE event
\begin{align}
\mathcal{E}_2
:=
\Bigg\{ 
\frac{1}{\kappa} \sum_{t=1}^{\kappa}\mathbf{z}_{k\kappa+t}\mathbf{z}_{k\kappa+t}^{\top}
\succeq a_1^2 \mathbf{I}_{n_z},
\forall\,0 \le k \le \left\lfloor\frac{T - 1}{\kappa}\right\rfloor - 1
\Bigg\},\label{eq:E2_eq}
\end{align}
Then $\mathbb{P}(\operatorname{diam}(\mathbb{S}_T)>\epsilon)
\le \mathbb{P}(\mathcal{E}_2^c)+\mathbb{P}(\mathcal{E}_1\cap \mathcal{E}_2)$ holds for any $\epsilon > 0$. We first bound $\mathbb{P}(\mathcal{E}_2^c)$. With $M=\sqrt{6C_zT^{r+1}/\eta}$, the truncation radius in \Cref{lem:E2c} satisfies
$\epsilon_z(M)=\min\{a_1^2/(2M^2),1/4\}=\epsilon_z$, and $C_zT^{r+1}/M^2=\eta/6$. \Cref{lem:E2c} gives
\begin{align*}
\mathbb{P}(\mathcal{E}_2^c)\le \frac{T}{\kappa}v_{\epsilon_z,n_z}\exp\!\left(-\frac{\kappa p_z^2}{8}\right)+\frac{\eta}{6}.
\end{align*}
Using the logarithmic covering bound $\log v_{\epsilon_z,n_z}\le \log 544+5n_z\log(n_z/\epsilon_z)$ from \Cref{lem:sphere-cover} and $\kappa p_z^2/8\ge L_z$, the first term is at most
\begin{align*}
T\exp\!\left(\log 544+5n_z\log\frac{n_z}{\epsilon_z}-L_z\right)
\le \frac{\eta}{3e}\le \frac{\eta}{3}.
\end{align*}
Hence $\mathbb{P}(\mathcal{E}_2^c)\le \eta/2$.

It remains to bound $\mathbb{P}(\mathcal{E}_1\cap \mathcal{E}_2)$. Apply \Cref{lem:E1E2} with $\delta=\epsilon$, and since $\epsilon\le 16\sqrt n\,\epsilon_0/(k_zp_z)$ by theorem hypothesis, and from \Cref{lem:E1E2}, $\bar\epsilon_\delta=\min\{\epsilon_\delta,\epsilon_0\} = \epsilon_\delta=\frac{a_1\epsilon}{4\sqrt n}
=\frac{k_zp_z\epsilon}{16\sqrt n}$. By \Cref{ass:tight-bound-w},
$q_w(\epsilon_\delta)\ge q:=c_wk_zp_z\epsilon/(16\sqrt n)$, and $q\le1$. \Cref{lem:E1E2} therefore yields
\begin{align*}
\mathbb{P}(\mathcal{E}_1\cap \mathcal{E}_2)\le v_\gamma(M)\exp\!\left(-q\left\lfloor\frac{T-1}{\kappa}\right\rfloor\right)+\frac{\eta}{6}.
\end{align*}
By the theorem condition on $T$, which is just,
\begin{align}
T\ge \frac{256\sqrt n}{k_zp_z^3c_w\epsilon}L_zL_\gamma
=
\frac{16L_z}{p_z^2}\frac{L_\gamma}{q},
\label{eq:thm1_condition_on_T}
\end{align} 
and by $\kappa\le16L_z/p_z^2$, we have $T/\kappa\ge L_\gamma/q$. Hence, $q\left\lfloor\frac{T-1}{\kappa}\right\rfloor
\ge q\left(\frac{T}{\kappa}-1\right) \ge L_\gamma-q
\ge L_\gamma-1$. Moreover, the choice of $M$ makes the net radius in \Cref{lem:E1E2} equal to the stated $\epsilon_\gamma$, and the covering bound gives
$\log v_\gamma(M)\le \log 544+5d_\gamma\log(d_\gamma/\epsilon_\gamma)$. Therefore,
\begin{align*}
v_\gamma(M)\exp\!\left(-q\left\lfloor\frac{T-1}{\kappa}\right\rfloor\right)
\le \exp\!\left(\log 544+5d_\gamma\log\frac{d_\gamma}{\epsilon_\gamma}-L_\gamma+1\right)
\le \frac{\eta}{3}.
\end{align*}
Thus, $\mathbb{P}(\mathcal{E}_1\cap \mathcal{E}_2)\le \eta/2$. If the condition \eqref{eq:thm1_condition_on_T} holds, then combining the two bounds gives $\mathbb{P}(\operatorname{diam}(\mathbb{S}_T)>\epsilon)\le \eta$. 
\end{proof}

\section{Simulation}
We consider a structured bilinear system such that $\mathbf{A}$ is diagonal and $\rho(\mathbf{A}) \le 1$, and the ${\mathbf{B}_i}$ are strictly lower triangular. This construction allows us to control the $\rho(\tilde{\mathbf A})$ of this class of bilinear systems. The entries of $\mathbf{A}$ and ${\mathbf{B}_i}$ are randomly generated and then uniformly scaled following the above conditions. We generate the input $\mathbf{u}_t$ \textit{i.i.d.} across time using a truncated Gaussian distribution, i.e., $\mathbf{u}_t \sim \mathcal{N}(\mathbf{0},\mathbf{I})$ truncated to $\{\mathbf{u}: \|\mathbf{u}\|_{\infty}\le 2\}$. The $\mathbf{w}_t$ is generated \textit{i.i.d.} across time from a standard Gaussian distribution truncated to the same range, namely, $\mathbf{w}_t \sim \mathcal{N}(\mathbf{0},\mathbf{I})$ truncated to $\{\mathbf{w}: \|\mathbf{w}\|_{\infty}\le 1\}$. The code used to reproduce the results, along with the matrices $\mathbf{A}$ and $\mathbf{B}_i$, is available on GitHub\footnote{\url{https://github.com/Hongyu-Yi/sys_id_bilinear}}.

We compare two estimators: SME and ordinary least squares (OLS), defined as $\mathbf{\widehat \Theta}_{\mathrm{OLS}} = \arg \min\nolimits_{\mathbf{\Theta}}\  1/2 \sum\nolimits_{t=1}^{T-1} \|\mathbf{x}_{t+1} - \mathbf{\Theta} \mathbf{z}_t\|_2^2$, which has closed-form solution $\mathbf{\widehat \Theta}_{\mathrm{OLS}} = \mathbf{Y} \mathbf{Z}^\top (\mathbf{Z} \mathbf{Z}^\top)^{-1}$ with $\mathbf{Y} = [\mathbf{x}_2,..,\mathbf{x}_T] \in \mathbb{R}^{n \times (T-1)}$ and $\mathbf{Z}=[\mathbf{z}_1,..,\mathbf{z}_{T-1}] \in \mathbb{R}^{(n+nm)\times (T-1)}$.

\begin{figure}[t]
    \centering
    \includegraphics[width=0.65\linewidth]{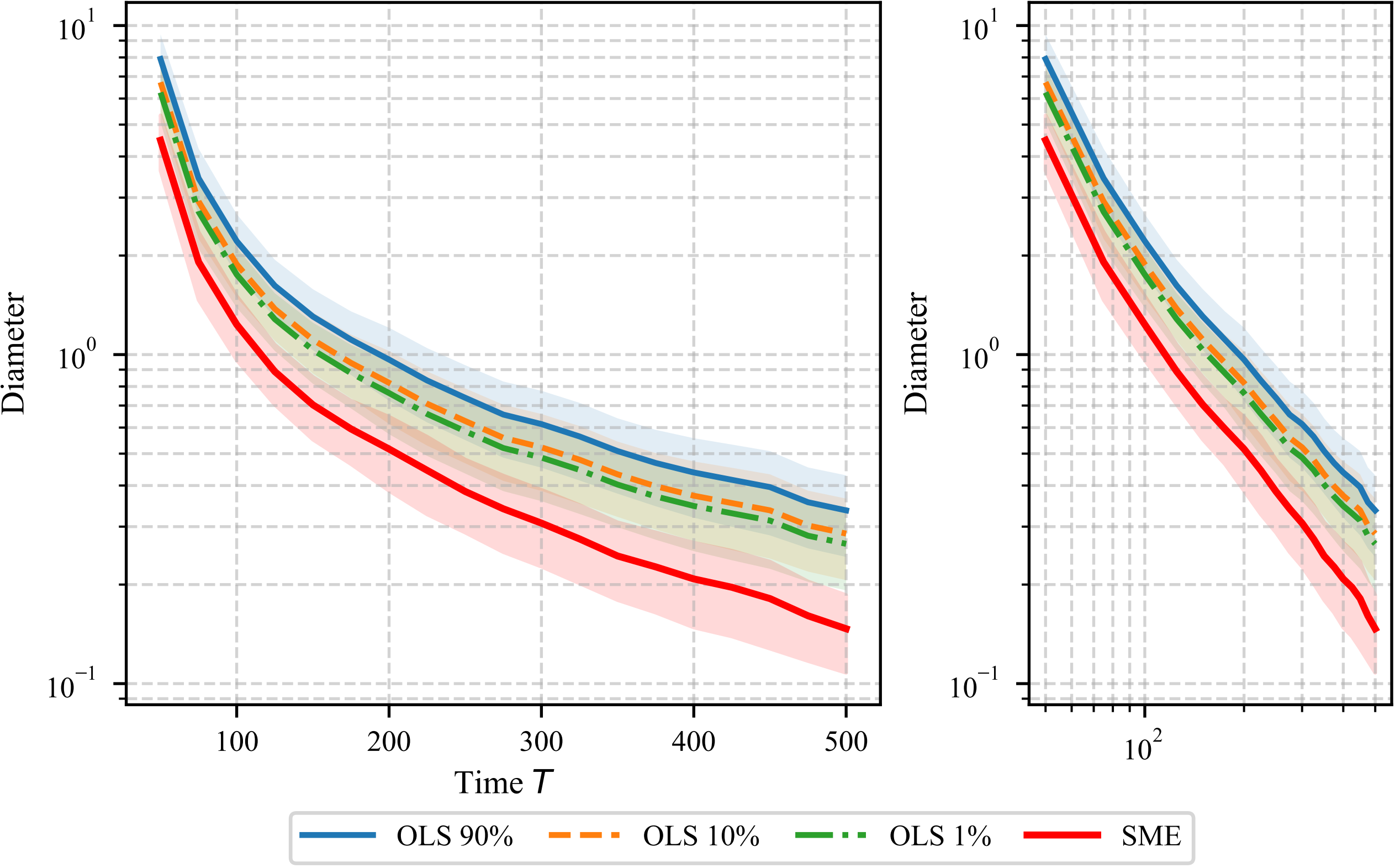}
    \caption{Diameters of Uncertainty Sets Contraction}
    \label{fig:diameter}
\end{figure}

In our simulation, we use the 90\%, 10\% and 1\% confidence regions of the OLS as the baseline uncertainty set constructed from the OLS point estimation $\mathbf{\widehat \Theta}_{\mathrm{OLS}}$. The OLS confidence region diameter\footnote{Here, we use probabilistic uncertainty quantification for OLS as baselines.} is computed as the minimum of the Lemma~E.3 bound in~\cite{simchowitz2020NaiveExploration} and a Wald-type asymptotic ellipsoid bound~\cite{draper1998applied,ljung1999system}. As for the diameter of the SME uncertainty set, we aim at calculating $\mathrm{diam}(\mathbb{S}_T)$, which is nonconvex, and we approximate the diameter as in \cite{Li2024SME-LTI}. We repeat the experiment 12 times and report the 90\% confidence interval. \Cref{fig:diameter} demonstrates that the SME uncertainty set diameter shrinks steadily with $T$, producing tighter uncertainty sets than all OLS-based baselines under different confidence.

\section{Conclusion}
This paper studied set-membership identification for bilinear systems with bounded noises. 
We showed that, despite the trajectory dependence of the bilinear regressor and possible polynomial state growth, the SME feasible set admits a finite-sample convergence guarantee with $\widetilde{\mathcal O}(1/\epsilon)$ sample complexity. 
Numerical results supported the theory and demonstrated clear advantages over an OLS-based baseline. Future work includes extending the finite-sample SME analysis to partially observed LTI systems, where process and measurement noises shape the feasible set and regressors come from input-output histories.

\section{Appendix}

\subsection{\Cref{lem:BMSB}: BMSB for $\mathbf{z}_t$}\label{app:lemBMSB}
\begin{proof}[Proof of \Cref{lem:BMSB}]
The condition that the process $\{\mathbf{z}_t\}_{t\ge 0}$ satisfies $(1,k_z^2 \mathbf{I},p_z)$-BMSB is equivalently expressed as the following: for any $j\geq1$ and a fixed $\mathbf{v} \in \R^{n+nm}$ with $\|\mathbf{v}\|_2=1$, $\mathbb{P}(|\langle \mathbf{v},\mathbf{z}_{j+1}\rangle| \geq k_z \|\mathbf{v}\|_2 | \mathcal{F}_j) \geq p_z$. To proceed with the proof, we write $\langle \mathbf{v},\mathbf{z}_{j+1}\rangle = \langle \mathbf{v}_0 + \mathbf{V} \mathbf{u}_{j+1},\mathbf{x}_{j+1}\rangle = \langle \mathbf{v}_0 + \mathbf{V} \mathbf{u}_{j+1},\mathbf{\Theta}_\star\mathbf{z}_{j}+\mathbf{w}_j\rangle$, where we define the decomposition $\mathbf{v}=[\mathbf{v}_0^\top,\mathbf{v}_1^\top,...,\mathbf{v}_m^\top]^{\top}\in \mathbb{R}^{n+nm}$ with $\mathbf{v}_i\in \mathbb{R}^n$ for all $i=0,\ldots,m$ and the concatenation matrix $\mathbf{V}=[\mathbf{v}_1,...,\mathbf{v}_m]\in \mathbb{R}^{n\times m}$. Thus, $\|\mathbf{v}_0\|_2^2+\|\mathbf{V}\|_F^2=1$. Now we define $3$ events:
\begin{align*}
\mathcal{E}_z &:= \left\{ \left| \left\langle \mathbf{v}_0+\mathbf{V}\mathbf{u}_{j+1},\mathbf{\Theta}_\star\mathbf{z}_j+ \mathbf{w}_{j} \right\rangle \right| \geq k_z \|\mathbf{v}\|_{2} \,\middle|\, \mathcal{F}_j \right\}, \\
\mathcal{E}_w &:= \left\{ \left| \left\langle \mathbf{v}_0 \!+\! \mathbf{V}\mathbf{u}_{j+1}, \mathbf{\Theta}_\star\mathbf{z}_j \!+\! \mathbf{w}_{j} \right\rangle \right|
    \!\ge\! k_0 \left\|\mathbf{v}_0 \!+\! \mathbf{V}\mathbf{u}_{j+1}\right\|_{2} \middle| \mathcal{F}_j \right\}, \\
\mathcal{E}_u &:= \left\{ \left\|\mathbf{v}_0 + \mathbf{V}\mathbf{u}_{j+1}\right\|_{2}
    \ge k_1 \|\mathbf{v}\|_{2} \,\middle|\, \mathcal{F}_j \right\}.
\end{align*}
Notice that the desired term is just $\mathbb{P}(\mathcal{E}_z)$. To lower bound the probability of this event, we will lower bound $\mathbb{P} (\mathcal{E}_u \cap \mathcal{E}_w)$ since $\mathcal{E}_u \cap \mathcal{E}_w \subseteq \mathcal{E}_z$ if $k_0k_1 \geq k_z$. In the following, we proceed to bound $\mathbb{P}(\mathcal{E}_w | \mathcal{E}_u)$ and $\mathbb{P}(\mathcal{E}_u)$.

\subsubsection{$\mathbb{P}(\mathcal{E}_w | \mathcal{E}_u)$} 
We have the following decomposition:
\begin{align*}
    &\left\langle \mathbf{v}_0 + \mathbf{V}\mathbf{u}_{j+1}, \mathbf{\Theta}_\star\mathbf{z}_j+ \mathbf{w}_{j} \right\rangle = \underbrace{\left\langle \mathbf{v}_0 + \mathbf{V}\mathbf{u}_{j+1}, \mathbf{\Theta}_\star\mathbf{z}_j\right\rangle}_{\text{Shift RV}} + \underbrace{\left\langle \mathbf{v}_0 + \mathbf{V}\mathbf{u}_{j+1}, \mathbf{w}_{j} \right\rangle}_{\text{Zero Mean RV}}.
\end{align*}
For notational simplicity, we define $\mathbf{q}_{j+1}:=\mathbf{v}_0 + \mathbf{V}\mathbf{u}_{j+1}$. Given any $\|\mathbf{q}\|_2 =1$, the pdf of $\langle\mathbf{q},\mathbf{w}_j\rangle$ is \textit{log-concave} and symmetric around $x=0$, implying that this pdf is unimodal and monotone on each side of its mode, i.e., non-decreasing to the left and non-increasing to the right. Define $\mathcal{F}_j^* = \sigma(\mathcal{F}_j \cup \sigma(\mathbf{q}_{j+1}))$, and therefore, under $\mathcal{F}_j^*$ we have
\begin{align}
    \mathbb{P}(\mathcal{E}_w) =& \mathbb{P}( \left| \left\langle \mathbf{q}_{j+1}, \mathbf{\Theta}_\star\mathbf{z}_j+ \mathbf{w}_{j} \right\rangle \right|
    \ge k_0 \left\|\mathbf{q}_{j+1}\right\|_{2}
    \,\mid\, \mathcal{F}_j^* )
    \geq \mathbb{P}( \left| \left\langle \mathbf{q}_{j+1},  \mathbf{w}_{j} \right\rangle \right|
    \geq k_0 \left\|\mathbf{q}_{j+1}\right\|_{2}
    \,\mid\, \mathcal{F}_j^* ) , \label{eq:lwb-Ew}
\end{align}
where we lower bound the conditional probability by dropping the scalar $\langle \mathbf{q}_{j+1}, \mathbf{\Theta}_\star\mathbf{z}_j\rangle$, yielding \eqref{eq:lwb-Ew}.
\Cref{fig:lwb-Ew} illustrates how this inequality holds. Precisely, shifting the interval from the center can only increase the corresponding outside probability.
\begin{figure}
    \centering
    \includegraphics[width=0.8\linewidth]{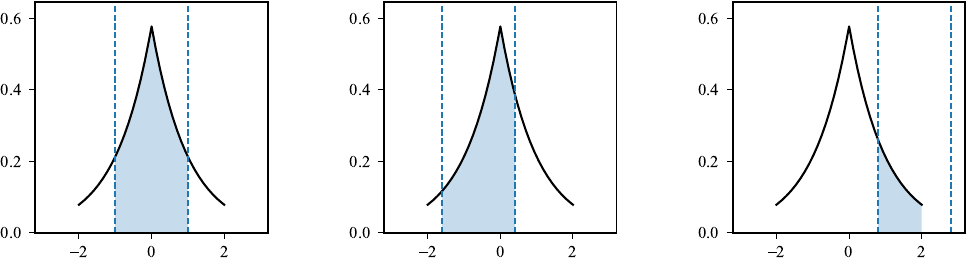}
    \caption{The shaded regions represent the probability $\mathbb{P}( \left| \left\langle \mathbf{q}_{j+1}, \mathbf{\Theta}\mathbf{z}_j+ \mathbf{w}_{j} \right\rangle \right|
    \leq k_0 \left\|\mathbf{q}_{j+1}\right\|_{2}
    \mid \mathcal{F}_j^* )$: (left) zero-centered case, yielding a symmetric and maximal in-window probability mass; (middle) a mild negative shift, decreasing the mass within the fixed window; (right) a shift toward the positive boundary, leaving only a small in-window mass.}
    \label{fig:lwb-Ew}
\end{figure}
The variance of the term $\left\langle \mathbf{q}_{j+1},  \mathbf{w}_{j} \right\rangle$ is lower bounded by $\text{Var}(\left\langle \mathbf{q}_{j+1},  \mathbf{w}_{j} \right\rangle) \geq \sigma_w^2\|\mathbf{q}_{j+1}\|_2^2$ under $\mathcal{F}_j^*$. Let $Y := \langle \mathbf{q}_{j+1}, \mathbf{w}_j \rangle$. We apply the Paley--Zygmund inequality to
$Y^2$, where for any $0<\theta<1$:
\begin{align*}
\mathbb{P}\left(Y^2 \ge \theta\mathbb{E}\big[Y^2 \mid \mathcal{F}_j^*\big]
  \middle|\, \mathcal{F}_j^* \right) \ge (1-\theta)^2 \frac{\big(\mathbb{E}[Y^2 \mid \mathcal{F}_j^*]\big)^2}{\mathbb{E}[Y^4 \mid \mathcal{F}_j^*]} .
\end{align*}
Moreover, $\mathbb{E}[Y^2 \mid \mathcal{F}_j^*] = \operatorname{Var}(Y \mid \mathcal{F}_j^*) \geq \sigma_w^2 \|\mathbf{q}_{j+1}\|_2^2 $.

Since $Y$ is one-dimensional symmetric log-concave, there exists
a universal constant $C>0$ such that $\mathbb{E}[Y^4 \mid \mathcal{F}_j^*]
\le
C \big(\mathbb{E}[Y^2 \mid \mathcal{F}_j^*]\big)^2$. Substituting this into the Paley--Zygmund inequality yields
\[\mathbb{P}\left(
  Y^2 \ge \theta\sigma_w^2 \|\mathbf{q}_{j+1}\|_2^2
  \middle| \mathcal{F}_j^*
\right)
\ge \mathbb{P}\left(Y^2 \ge \theta\mathbb{E}\big[Y^2 \mid \mathcal{F}_j^*\big]
  \middle|\, \mathcal{F}_j^* \right)\ge
(1-\theta)^2/C.\]
The left-hand side probability term can be written as $\mathbb{P}\left(|Y| \ge \sqrt{\theta}\sigma_w \|\mathbf{q}_{j+1}\|_2\,\middle| \mathcal{F}_j^* \right)$. Choosing the constant
$\theta = k_0^2 / \sigma_w^2 \in (0,1)$ and $k_0 \in (0, \sigma_w)$, we obtain $\mathbb{P}\!\left(
  |\langle \mathbf{q}_{j+1}, \mathbf{w}_j\rangle|
  \ge k_0 \|\mathbf{q}_{j+1}\|_2
  \middle|\, \mathcal{F}_j^*
\right)
\ge
(1 - k_0^2 / \sigma_w^2)^2/C
=: p_w$, which provides a constant lower bound for the probability. Corollary 4.1 in \cite{marsiglietti2025logconcave-consen} shows that under our settings of log-concavity and symmetry with 0, $C=6$. Therefore, we have
\begin{align}
    \mathbb{P}\!\left(
  |\langle \mathbf{q}_{j+1}, \mathbf{w}_j\rangle|
  \ge k_0 \|\mathbf{q}_{j+1}\|_2
  \middle|\, \mathcal{F}_j^*
\right) \ge \big(1 - k_0^2 / \sigma_w^2\big)^2 /6,\label{prop1-ineq}
\end{align}
with $k_0 < \sigma_w$, and $p_w =: (1 - k_0^2 / \sigma_w^2)^2/6$.

We consider the probability $\mathbb{P}(\mathcal{E}_w | \mathcal{E}_u,\mathcal{F}_j)$. This term can be lower bounded as follows
\[\mathbb{P}(\mathcal{E}_w | \mathcal{E}_u,\mathcal{F}_j) 
\overset{(i)}{=} \mathbb{E}[\mathbb{P}(\mathcal{E}_w | \mathcal{E}_u, \mathcal{F}_j^\ast) | \mathcal{E}_u,\mathcal{F}_j]
\overset{(ii)}{=} \mathbb{E}[\mathbb{P}(\mathcal{E}_w |  \mathcal{F}_j^\ast) | \mathcal{E}_u,\mathcal{F}_j]
\overset{(iii)}{\geq}  \mathbb{E}[p_w | \mathcal{E}_u,\mathcal{F}_j]= \big(1 - k_0^2 / \sigma_w^2\big)^2/ 6,\]
because $(i)$: the tower rule; $(ii)$: $\mathcal{E}_u\in\mathcal{F}_j^\ast$ since $\mathcal{E}_u$ is a function of $\mathbf{q}_{j+1}$ only; $(iii)$: the uniform bound from \eqref{prop1-ineq} under $\mathcal{F}_j^\ast$ for every $\mathbf{q}_{j+1}$.

\subsubsection{$\mathbb{P}(\mathcal{E}_u)$} 
We prove a uniform lower bound for $\mathcal{E}_u=\left\{\|\mathbf{q}_{j+1}\|_2\ge k_1\|\mathbf v\|_2\right\}$. Conditioning on $\mathcal F_j$, the random input $\mathbf{u}_{j+1}$ is independent of
$\mathcal F_j$, has mean zero, satisfies
$\mathbb E[\mathbf{u}_{j+1} \mathbf{u}_{j+1}^\top]=\sigma_u^2 \mathbf{I}_{m}$, and
$\|\mathbf{u}_{j+1}\|_\infty\le u_{\max}$ almost surely.

Let $X:=\|\mathbf{q}_{j+1}\|_2^2$. Using $\mathbb E[\mathbf{u}_{j+1}]=0$, we have $\mathbb E[X\mid \mathcal F_j]
=
\|\mathbf{v}_0\|_2^2
+
\mathbb E[\mathbf{u}_{j+1}^\top \mathbf{V}^\top \mathbf{V} \mathbf{u}_{j+1}]
=
\|\mathbf{v}_0\|_2^2+\sigma_u^2\|\mathbf{V}\|_F^2$. Therefore, $\mathbb E[X\mid \mathcal F_j]
\ge
\min\{1,\sigma_u^2\} \bigl(\|\mathbf{v}_0\|_2^2+\|\mathbf{V}\|_F^2\bigr)
=
\min\{1,\sigma_u^2\}$ holds. Moreover, since $\|\mathbf{u}_{j+1}\|_2^2\le m u_{\max}^2$, by Cauchy's inequality, we have 
\[\|\mathbf{v}_0+ \mathbf{V} \mathbf{u}_{j+1}\|_2^2
\le
\left(\|\mathbf{v}_0\|_2+\|\mathbf{V}\|_F\|\mathbf{u}_{j+1}\|_2\right)^2
\le
1+m u_{\max}^2.\] 
Then $0\le X\le 1+m u_{\max}^2$ almost surely, and hence $\mathbb E[X^2\mid\mathcal F_j]
\le
(1+m u_{\max}^2) \mathbb E[X\mid\mathcal F_j]$. Applying the Paley--Zygmund inequality to $X$, for any $\theta\in(0,1)$, the lower bound $\mathbb P\left(
X\ge \theta \mathbb E[X\mid\mathcal F_j]\mid\mathcal F_j
\right)
\ge
(1-\theta)^2
\frac{(\mathbb E[X\mid\mathcal F_j])^2}
{\mathbb E[X^2\mid\mathcal F_j]} $ is satisfied.

Using the above upper bound on $\mathbb E[X^2\mid\mathcal F_j]$, we obtain $\mathbb P\left(
X\ge \theta \mathbb E[X\mid\mathcal F_j]\mid\mathcal F_j
\right)
\ge
(1-\theta)^2\frac{\min\{1,\sigma_u^2\}}{1+m u_{\max}^2}$. Since $X\ge \theta \mathbb E[X\mid\mathcal F_j]
 \Longrightarrow 
X\ge \theta \min\{1,\sigma_u^2\}$, we have $\mathbb P(
\|\mathbf{v}_0+ \mathbf{V} \mathbf{u}_{j+1}\|_2
\ge
\sqrt{\theta \min\{1,\sigma_u^2\}}
\mid \mathcal F_j
)
\ge
(1-\theta)^2\frac{\min\{1,\sigma_u^2\}}{1+m u_{\max}^2}$. Choosing $\theta=1/2$, we obtain the lower bound $\mathbb P(\mathcal{E}_u\mid\mathcal F_j)
\ge
\frac{\min\{1,\sigma_u^2\}}{4 (1+m u_{\max}^2)}
=:p_u$, with a conservative $k_1:= \frac{\sqrt{\min\{1,\sigma_u^2\}}}{2},\ p_u:=\frac{\min\{1,\sigma_u^2\}}{4(1+m u_{\max}^2)} $.

\subsubsection{Obtain BMSB condition} Now, we combine the above results to get the BMSB condition: 
    $\mathbb{P}(\mathcal{E}_z | \mathcal{F}_j) \geq \frac{3\min\{1,\sigma_u^2\}}{128(1+m u_{\max}^2)}
 \coloneq p_z$, and thus $\{\mathbf{z}_t\}_{t\ge 0}$ satisfies $(1,k_z^2 \mathbf{I}_{n+nm},p_z)$-BMSB with $k_z=k_0k_1$, and $k_0 =  \sigma_w/2, k_1 = \sqrt{\min\{1,\sigma_u^2\}}/2$.
\end{proof}

\subsection{\Cref{lem:E2c}: Bound on $\mathbb P(E_2^c)$}
\label{AppA}

\begin{lemma}[Bound on $\mathbb P(E_2^c)$]\label{lem:E2c}
Let $n_z:=n+nm$ and $a_1:=k_zp_z/4$. Under \Cref{ass:input,ass:noise} and \Cref{lem:PMS,lem:BMSB}, for any $M>0$ and any $\kappa\in\mathbb N_+$, define
$\varepsilon_z(M):=\min\{a_1^2/(2M^2),1/4\}$. Then
\begin{align*}
\mathbb P(E_2^c)
\le
\frac{T}{\kappa} v_{\varepsilon_z,n_z}(M)
\exp\!\left(-\frac{\kappa p_z^2}{8}\right)
+
\frac{C_zT^{r+1}}{M^2},
\end{align*}
where $v_{\varepsilon_z,n_z}(M)$ is the cardinality of an $\varepsilon_z(M)$-net of the unit sphere in $\mathbb R^{n_z}$, and can be upper bounded by
\begin{align*}
v_{\varepsilon_z,n_z}(M)
\le
544\,n_z^{2.5}\log\!\left(\frac{2 n_z}{\varepsilon_z(M)}\right)
\left(\frac{2}{\varepsilon_z(M)}\right)^{n_z}.
\end{align*}
In particular, $\log v_{\epsilon_z,n_z}(M)
\le
\log 544 + 5n_z\log\left(\frac{n_z}{\epsilon_z(M)}\right)$.
The constant $C_z<\infty$ depends only on $c_{\rm PMS},r,m,u_{\max}$.
\end{lemma}

To formally prove \Cref{lem:E2c}, we first introduce the following useful lemmas.

\begin{lemma}[Finite covering of unit ball, Thm. D.1 in \cite{Li2024SME-LTI}]\label{lem:finite-cover}
Let $v_{\varepsilon,n}$ be the minimal number of closed Euclidean balls of radius $\varepsilon$ needed to cover $\overline{B}_n(0,1)$.
For every integer $n\ge 1$ and every $0<\varepsilon<0.5$, we have $v_{\varepsilon,n}
\le 544\, n^{2.5}\log\!\left({n}/{\varepsilon}\right)\left({1}/{\varepsilon}\right)^n,$ and the centers can be chosen inside $\overline{B}_n(0,1)$.
\end{lemma}

\begin{lemma}[Finite covering of unit spheres]
\label{lem:sphere-cover}
Let \(v_{\epsilon,d}^{\mathsf{sph}}\) be the minimal number of closed Euclidean balls
of radius \(\epsilon\), with centers on the unit sphere
\(\mathbb S^{d-1}:=\{x\in\mathbb R^d:\|x\|_2=1\}\), needed to cover
\(\mathbb S^{d-1}\). Then, for every integer \(d\ge 1\) and every
\(0<\epsilon\le 1/4\), it satisfies $v_{\epsilon,d}^{\mathsf{sph}}
\le
544 d^{2.5}
\log\left(\frac{2d}{\epsilon}\right)
\left(\frac{2}{\epsilon}\right)^d$. Moreover, $\log v_{\epsilon,d}^{\mathsf{sph}}
\le
\log 544 + 5d\log\left(\frac{d}{\epsilon}\right)$.
\end{lemma}

\begin{proof}[Proof of \Cref{lem:sphere-cover}]
Apply the unit-ball covering result of \Cref{lem:finite-cover}
with radius \(\epsilon/2\). Thus the closed unit ball in \(\mathbb R^d\)
can be covered by at most $544 d^{2.5}\log\left(\frac{2d}{\epsilon}\right)\left(\frac{2}{\epsilon}\right)^d$ closed balls of radius \(\epsilon/2\). For each covering ball that intersects
\(\mathbb S^{d-1}\), choose one point from the intersection as its center and
discard all covering balls that do not intersect \(\mathbb S^{d-1}\). Then, for
any \(x\in \mathbb S^{d-1}\), \(x\) lies in one of the original radius-\(\epsilon/2\)
balls, and the selected point from the same ball is also on \(\mathbb S^{d-1}\).
Hence their distance is at most \(\epsilon\). This gives the required covering bound. For the logarithmic bound, since \(0<\epsilon\le 1/4\) and \(d\ge 1\), it satisfies $\log\left(
544 d^{2.5}
\log\left(2d/\epsilon\right)
\left(2/\epsilon\right)^d
\right)
\le
\log 544 + 5d\log\left(\frac d\epsilon\right)$.
\end{proof}

\begin{lemma}[Small-ball lower tail, Prop. 2.5 in \cite{Simchowitz2018LearningWithoutMixing}]\label{lem:Simchowitz}
Let $\{Z_t\}_{t\ge 1}$ be a real-valued process adapted to $\{\mathcal{F}_t\}_{t\ge 1}$.
If $\{Z_t\}_{t\ge 1}$ satisfies the $(1,k,p)$-BMSB condition, then for every integer $T\ge 1$, the inequality $\mathbb{P}(\sum\nolimits_{t=1}^T Z_t^2 \le \frac{k^2p^2}{8}\,T)
\le \exp(-\frac{Tp^2}{8})$ holds.
\end{lemma}

\begin{proof}[Proof of \Cref{lem:E2c}]
Recall that $\mathcal{E}_2$ is defined in \eqref{eq:E2_eq}
and $\mathcal{E}_{2,M}:=\{\max_{1\le t\le T}\|\mathbf{z}_t\|_2\le M\}$. Since $\mathbf{z}_t=[\mathbf{x}_t^\top,(\mathbf{u}_t\otimes \mathbf{x}_t)^\top]^\top$, we have
$\|\mathbf{z}_t\|_2^2=(1+\|\mathbf{u}_t\|_2^2)\|\mathbf{x}_t\|_2^2\le (1+mu_{\max}^2)\|\mathbf{x}_t\|_2^2$. By \Cref{lem:PMS},
$\mathbb E\|\mathbf{x}_t\|_2^2\le c_{\rm PMS}(1+t^r)$, hence,
$\mathbb E\|\mathbf{z}_t\|_2^2\le C_z(1+t^r)$. Therefore Markov's inequality and a union bound give
\begin{align*}
\mathbb P(\mathcal{E}_{2,M}^c)
\le
\sum\nolimits_{t=1}^{T}\frac{\mathbb E\|\mathbf{z}_t\|_2^2}{M^2}
\le
\frac{C_zT^{r+1}}{M^2}.
\end{align*}

It remains to bound $\mathbb P(\mathcal{E}_2^c\cap \mathcal{E}_{2,M})$. Fix a block index $k$, and write $\mathbf{K}_k:=1/ \kappa \sum\nolimits_{t=1}^{\kappa}
\mathbf{z}_{k\kappa+t}\mathbf{z}_{k\kappa+t}^{\top}$. Let $\mathcal V=\{\lambda_1,\ldots,\lambda_{v_{\varepsilon_z,n_z}}\}$ be an $\varepsilon_z(M)$-net of the unit sphere in $\mathbb R^{n_z}$. The bound for $v_{\varepsilon_z,n_z}$ follows from \Cref{lem:sphere-cover} applied to dimension $n_z$.

For any fixed unit vector $\lambda_i \in \mathcal{V}$, \Cref{lem:BMSB} implies that the scalar process $\lambda_i^\top \mathbf{z}_t$ satisfies the $(1,k_z,p_z)$-BMSB condition. Applying \Cref{lem:Simchowitz} on the block $k\kappa+1,\ldots,k\kappa+\kappa$, conditionally on $\mathcal F_{k\kappa}$, yields
\begin{align*}
\mathbb P\left(
\lambda_i^\top \mathbf{K}_k\lambda_i
\le
\frac{k_z^2p_z^2}{8}
\,\middle|\,
\mathcal F_{k\kappa}
\right)
\le
\exp\!\left(-\frac{\kappa p_z^2}{8}\right).
\end{align*}
Since $k_z^2p_z^2/8=2a_1^2$, a union bound over the net $\mathcal{V}$ gives
\begin{align*}
\mathbb P\left(
\exists i\in[v_{\varepsilon_z,n_z}]:
\lambda_i^\top \mathbf{K}_k\lambda_i\le 2a_1^2
\,\middle|\,
\mathcal F_{k\kappa}
\right)
\le
v_{\varepsilon_z,n_z}(M)
\exp\!\left(-\frac{\kappa p_z^2}{8}\right).
\end{align*}

We now show that the net condition implies the full matrix lower bound on the truncation event. On $\mathcal{E}_{2,M}$, every regressor $\mathbf{z}$ in the block satisfies $\|\mathbf{z}_{k\kappa+t}\|_2\le M$, and thus $\|\mathbf{K}_k\|_2\le M^2$. Suppose that $\lambda_i^\top \mathbf{K}_k\lambda_i>2a_1^2$ for every $\lambda_i\in\mathcal V$. For any unit vector $\mathbf{v} \in \R^{n_z}$, choose $\lambda_i\in\mathcal V$ such that $\|\mathbf{v}-\lambda_i\|_2\le\varepsilon_z(M)$. Since $\mathbf{K}_k\succeq0$,
\begin{align*}
\mathbf{v}^\top \mathbf{K}_k \mathbf{v}
=
\lambda_i^\top \mathbf{K}_k\lambda_i
+
(\mathbf{v}-\lambda_i)^\top \mathbf{K}_k(\mathbf{v}+\lambda_i)
\ge
2a_1^2-2\varepsilon_z(M)M^2
\ge
a_1^2.
\end{align*}
The last inequality follows from the definition of $\varepsilon_z(M)$: if $\varepsilon_z(M)=a_1^2/(2M^2)$, it is equality; if $\varepsilon_z(M)=1/4$, then $M^2\le2a_1^2$, so $2\varepsilon_z(M)M^2\le a_1^2$. Hence $\mathbf{K}_k\succeq a_1^2 \mathbf{I}_{n_z}$.

Therefore, on $\mathcal{E}_{2,M}$, the failure of the $k$-th block PE implies that at least one net direction has small quadratic value:
\begin{align*}
\{\mathbf{K}_k\not\succeq a_1^2 \mathbf{I}_{n_z}\}\cap \mathcal{E}_{2,M}
\subseteq
\left\{
\exists i\in[v_{\varepsilon_z,n_z}]:
\lambda_i^\top \mathbf{K}_k\lambda_i\le 2a_1^2
\right\}.
\end{align*}
Consequently, $\mathbb P\left(\{\mathbf{K}_k\not\succeq a_1^2 \mathbf{I}_{n_z}\}\cap \mathcal{E}_{2,M}\right)
\le
v_{\varepsilon_z,n_z}(M)
\exp(-\kappa p_z^2 / 8)$. Finally, there are at most $T/\kappa$ full blocks appearing in the definition of $\mathcal{E}_2$. A union bound over these blocks gives $\mathbb P(\mathcal{E}_2^c\cap \mathcal{E}_{2,M})
\le
\frac{T}{\kappa}
v_{\varepsilon_z,n_z}(M)
\exp\!\left(-\frac{\kappa p_z^2}{8}\right)$. Combining this with the truncation bound,
\begin{align*}
\mathbb P(\mathcal{E}_2^c)
\le
\mathbb P(\mathcal{E}_2^c\cap \mathcal{E}_{2,M})+\mathbb P(\mathcal{E}_{2,M}^c)
\le
\frac{T}{\kappa}
v_{\varepsilon_z,n_z}(M)
\exp\!\left(-\frac{\kappa p_z^2}{8}\right)
+
\frac{C_zT^{r+1}}{M^2},
\end{align*}
which completes the proof.
\end{proof}

\subsection{\Cref{lem:E1E2}: Bound on $\mathbb{P}(\mathcal{E}_1\cap\mathcal{E}_2)$}
\label{AppB}
\begin{lemma}[Bound on $\mathbb{P}(\mathcal{E}_1\cap\mathcal{E}_2)$]
\label{lem:E1E2}
For any $M>0$ and any $\delta>0$, define $\varepsilon_\delta:=\frac{a_1\delta}{4\sqrt n},\
\bar\varepsilon_\delta:=\min\{\varepsilon_\delta,\varepsilon_0\}$ with $\varepsilon_0$ from \Cref{ass:tight-bound-w}, and $q_w(\varepsilon):=\inf_{j\in[n],\,b\in\{\pm1\}}
\mathbb P\!\left(b\mathbf{w}_t[j]\ge w_{\max} - \varepsilon\right),
\ 0<\varepsilon\le \varepsilon_0$. Then, $\mathbb{P}(\mathcal{E}_1\cap\mathcal{E}_2)
\le
v_\gamma(M)\Bigl(1-q_w(\bar\varepsilon_\delta)\Bigr)^{\lfloor (T-1)/\kappa\rfloor}
+\frac{C_z T^{r+1}}{M^2}$,
where $v_\gamma(M)$ satisfies:
\begin{align*}
v_\gamma(M)
&\le
544 (n^2+n^2 m)^{2.5}\,
\log\Big(\frac{2n^2+ 2n^2 m}{\epsilon_\gamma}\Big)\,
\Big(\frac{2}{\epsilon_\gamma}\Big)^{n^2+n^2 m},
\end{align*}
with $\epsilon_\gamma := \min\left\{\frac{a_1}{4M\sqrt n},\frac{1}{4} \right\}$. Also, $\log v_\gamma(M)
\le
\log 544 + (5 n^2+ 5 n^2 m)\log\left(\frac{n^2+n^2 m}{\epsilon_\gamma}\right)$.
\end{lemma}

\begin{proof}
We bound $\mathbb{P}(\mathcal{E}_1\cap \mathcal{E}_2)$ by combining a finite-net reduction over error directions with a block-wise elimination argument induced by $\mathcal{E}_2$.
We first truncate the regressors to control approximation errors, then discretize the unit Frobenius sphere to consider finitely many candidate directions, and finally show that under $\mathcal{E}_2$ each block rules out every such direction unless the disturbance avoids a corresponding boundary neighborhood of $\mathbb{W}$.

Using the definition of event $\mathcal{E}_{2,M}$, we have
\begin{align}
\mathbb{P}(\mathcal{E}_1\cap \mathcal{E}_2)
\le
\mathbb{P}(\mathcal{E}_1\cap \mathcal{E}_2\cap \mathcal{E}_{2,M})
+
\mathbb{P}(\mathcal{E}_{2,M}^c).
\label{eq:L6_split_trunc}
\end{align}
The tail term $\mathbb{P}(\mathcal{E}_{2,M}^c)$ is bounded as in \Cref{lem:E2c}, yielding $\mathbb{P}(\mathcal{E}_{2,M}^c)\le C_zT^{r+1}/M^2$.

Next, recall $\mathcal{E}_1=\{\exists\,\gamma\in\Gamma_T:\|\gamma\|_F\ge \delta/2\}$ and define the unit Frobenius sphere $\mathcal{S}:=\{\Gamma\in\mathbb{R}^{n\times n_z}:\ \|\Gamma\|_F=1\},\ n_z:=n+nm$. Let $\mathcal{M}:=\{\Gamma_1,\ldots,\Gamma_{v_\gamma(M)}\}$ be an $\varepsilon_\gamma$-net of $\mathcal{S}$ in Frobenius norm and set $\varepsilon_\gamma:=\min\{\frac{a_1}{4M\sqrt{n}},\frac{1}{4}\}, a_1:=\frac{k_zp_z}{4}$. For each net point $\Gamma_i$, define
\begin{align*}
\mathcal{A}_i
:=
\Big\{
\exists\,\gamma\in\Gamma_T:\ \|\gamma\|_F\ge \delta/2,\ 
\big\|\gamma/\|\gamma\|_F-\Gamma_i\big\|_F\le \varepsilon_\gamma
\Big\}.
\end{align*}
On $\mathcal{E}_{2,M}$, the existence of a large feasible $\gamma$ implies that its normalized direction lies within $\varepsilon_\gamma$ of some net point, hence
$\mathcal{E}_1\cap \mathcal{E}_{2,M}\subseteq \bigcup_{i=1}^{v_\gamma(M)} \mathcal{A}_i$, and
\begin{align}
\mathbb{P}(\mathcal{E}_1\cap \mathcal{E}_2\cap \mathcal{E}_{2,M})
\le
\sum\nolimits_{i=1}^{v_\gamma(M)}
\mathbb{P}(\mathcal{A}_i\cap \mathcal{E}_2\cap \mathcal{E}_{2,M}).
\label{eq:E1_net_union}
\end{align}

It remains to bound $\mathbb{P}(\mathcal{A}_i\cap \mathcal{E}_2\cap \mathcal{E}_{2,M})$ for a fixed net direction $\Gamma_i$. Given $\mathbf z_t$ is $\mathcal F_t$-measurable and $\mathbf w_t$ is independent of $\mathcal F_t$, for each $i$ and each time $t$, choose $j_{i,t}\in[n]$ and $b_{i,t}\in\{\pm1\}$ such that $\|\Gamma_i\mathbf z_t\|_\infty = b_{i,t}(\Gamma_i\mathbf z_t)[j_{i,t}]$. Then $j_{i,t}$ and $b_{i,t}$ are $\mathcal F_t$-measurable. For each block $k$, define the stopping time
\begin{align*}
L_{i,k}:=
\min\left(
\{\ell \in \{1,\ldots,\kappa\} :
\|\Gamma_i\mathbf z_{k\kappa+\ell}\|_\infty \ge \frac{a_1}{\sqrt{n}}\}
\cup \{\kappa + 1\}
\right).
\end{align*}
For every $\ell\le \kappa$, the event $\{L_{i,k}=\ell\}$ is determined by $\mathbf z_{k\kappa+1},\ldots,\mathbf z_{k\kappa+\ell}$, hence $\{L_{i,k}=\ell\}\in\mathcal F_{k\kappa+\ell}$.We recall the block Gram matrix $\mathbf K_k:=\frac{1}{\kappa}\sum_{t=1}^{\kappa}\mathbf z_{k\kappa+t}\mathbf z_{k\kappa+t}^\top$, so that $\mathcal E_2=\{\mathbf K_k\succeq a_1^2 \mathbf{I}_{n_z}\ \text{for all }k\}$.
On $\mathcal E_2$, for every $i,k$, we have $1/\kappa\sum\nolimits_{t=1}^{\kappa}\|\Gamma_i\mathbf z_{k\kappa+t}\|_2^2
=
\operatorname{tr}(\Gamma_i\mathbf K_k\Gamma_i^\top)
\ge a_1^2\|\Gamma_i\|_F^2=a_1^2$. Therefore, some $t\in\{1,\ldots,\kappa\}$ satisfies $\|\Gamma_i\mathbf z_{k\kappa+t}\|_\infty\ge a_1/\sqrt n$, and thus $L_{i,k}\le\kappa$ on $\mathcal E_2$.

Now consider $\mathcal A_i\cap\mathcal E_2\cap\mathcal E_{2,M}$.
By definition of $\mathcal A_i$, there exists $\gamma\in\Gamma_T$ with $\|\gamma\|_F\ge \delta/2$ and
$\|\gamma/\|\gamma\|_F-\Gamma_i\|_F\le \varepsilon_\gamma$.
Let $t_{i,k}:=k\kappa+L_{i,k}$. On $\mathcal E_2$, $L_{i,k}\le\kappa$, and on $\mathcal E_{2,M}$, $\|\mathbf z_{t_{i,k}}\|_2\le M$. Therefore,
\begin{align*}
&b_{i,t_{i,k}}
\left(\frac{\gamma}{\|\gamma\|_F}\mathbf z_{t_{i,k}}\right)[j_{i,t_{i,k}}]
\ge
b_{i,t_{i,k}}(\Gamma_i\mathbf z_{t_{i,k}})[j_{i,t_{i,k}}]
-
\left\|
\left(\frac{\gamma}{\|\gamma\|_F}-\Gamma_i\right)\mathbf z_{t_{i,k}}
\right\|_\infty  \ge
\frac{a_1}{\sqrt n}-\varepsilon_\gamma M
\ge
\frac{3a_1}{4\sqrt n}.
\end{align*}

Multiplying by $\|\gamma\|_F\ge\delta/2$ gives, for every block $k$,
\begin{align*}
b_{i,t_{i,k}}(\gamma\mathbf z_{t_{i,k}})[j_{i,t_{i,k}}]
\ge
a_1\delta / (4\sqrt n)
=:\varepsilon_\delta \ge \bar\varepsilon_\delta.
\end{align*}

By \Cref{ass:tight-bound-w}, $q_w(\varepsilon)\ge c_w\varepsilon$ for
$0<\varepsilon\le \varepsilon_0$, and for every fixed
$j\in[n]$ and $b\in\{\pm1\}$, $\mathbb P\!\left(b \mathbf{w}_t[j]\ge -w_{\max}+\bar\varepsilon_\delta\right)
\le 1-q_w(\bar\varepsilon_\delta)$.

For each $i,k$, define the block non-elimination event $\mathcal G_{i,k}
:=
\Big\{
L_{i,k}\le\kappa,\
b_{i,t_{i,k}}\mathbf w_{t_{i,k}}[j_{i,t_{i,k}}]
\ge
-w_{\max}+\bar\varepsilon_\delta
\Big\}$. We claim that
\begin{align*}
\mathcal A_i\cap \mathcal E_2\cap \mathcal E_{2,M}
\subseteq
\bigcap\nolimits_{k=0}^{\lfloor (T-1)/\kappa\rfloor-1}\mathcal G_{i,k}.
\end{align*}
At $t=t_{i,k}$, with $b=b_{i,t_{i,k}}$ and $j=j_{i,t_{i,k}}$, feasibility gives
$-w_{\max}\le b(\mathbf w_t[j]-(\gamma\mathbf z_t)[j])\le w_{\max}$.
Since $b(\gamma\mathbf z_t)[j] \ge\varepsilon_\delta \ge \bar\varepsilon_\delta$, the left inequality yields
$b\mathbf w_t[j]\ge -w_{\max}+\bar\varepsilon_\delta$, proving the claim.

It remains to bound the probability of the intersection of the events $\mathcal G_{i,k}$. Let
$\mathcal H_{i,k}:=\cap_{s=0}^{k-1}\mathcal G_{i,s}$, with $\mathcal H_{i,0}$ being the whole sample space.
For any $\ell\in\{1,\ldots,\kappa\}$ and $t=k\kappa+\ell$, the event
$\mathcal H_{i,k}\cap\{L_{i,k}=\ell\}$ is $\mathcal F_t$-measurable. Moreover, $b_{i,t}$ and $j_{i,t}$ are $\mathcal F_t$-measurable, while $\mathbf w_t$ is independent of $\mathcal F_t$. Hence,
\begin{align*}
\mathbb P\!\left(
\mathcal H_{i,k}\cap\{L_{i,k}=\ell\}
\cap
\{b_{i,t}\mathbf w_t[j_{i,t}]
\ge -w_{\max}+\bar\varepsilon_\delta\}
\right) \le
(1-q_w(\bar\varepsilon_\delta))
\mathbb P\!\left(\mathcal H_{i,k}\cap\{L_{i,k}=\ell\}\right).
\end{align*}

Summing over $\ell=1,\ldots,\kappa$ gives $\mathbb P(\mathcal H_{i,k}\cap\mathcal G_{i,k}) \le (1-q_w(\bar\varepsilon_\delta))\mathbb P(\mathcal H_{i,k})$. Iterating over the blocks yields
\begin{align*}
\mathbb P\!\left(
\bigcap\nolimits_{k=0}^{\lfloor (T-1)/\kappa\rfloor-1}\mathcal G_{i,k}
\right)
\le
\big(1-q_w(\bar\varepsilon_\delta)\big)^{\lfloor (T-1)/\kappa\rfloor}.
\end{align*}
Combining this with the net union bound \eqref{eq:E1_net_union}, we obtain
\begin{align*}
\mathbb P(\mathcal E_1\cap\mathcal E_2\cap\mathcal E_{2,M})
\le
v_\gamma(M)\big(1-q_w(\bar\varepsilon_\delta)\big)^{\lfloor (T-1)/\kappa\rfloor}.
\end{align*}
Together with \eqref{eq:L6_split_trunc} and $\mathbb P(\mathcal E_{2,M}^c)\le C_zT^{r+1}/M^2$, we obtain \Cref{lem:E1E2}.
\end{proof}

\bibliographystyle{ieeetr}
\bibliography{ref.bib}

\end{document}